\documentclass{article}
\usepackage{iclr2027_conference,times}
\usepackage[T1]{fontenc} %

\usepackage{amsmath,amsfonts,bm}

\def\eqref#1{equation~\ref{#1}}

\def\1{\bm{1}}

\DeclareMathAlphabet{\mathsfit}{\encodingdefault}{\sfdefault}{m}{sl}
\SetMathAlphabet{\mathsfit}{bold}{\encodingdefault}{\sfdefault}{bx}{n}

\usepackage{amsthm}
\usepackage{booktabs}
\usepackage{multirow}
\usepackage{caption}
\usepackage{subcaption}
\usepackage{graphicx}
\usepackage{tabularx}
\usepackage{longtable}
\usepackage{float}
\usepackage{flafter}
\usepackage{placeins}
\usepackage{wrapfig}
\usepackage{xcolor}
\usepackage{colortbl}
\definecolor{oursrow}{RGB}{235,243,250}
\usepackage{hyperref}
\usepackage{url}

\newcommand{\method}{\textsc{SharpDraft}} %

\newcolumntype{Y}{>{\raggedright\arraybackslash}X}
\newtheorem{proposition}{Proposition}
\DeclareMathOperator{\logit}{logit}

\title{\method{}: Accelerating\\
Long-Context Speculative Decoding\\
with Cardinality-Aware Query Scaling}
\author{%
Jeonghoon Park\thanks{Equal contribution.}\quad
Seongwoon Jo\footnotemark[1]\quad
Jongwon Lee\footnotemark[1]\quad
Taesik Gong\thanks{Corresponding author.}\\
{\normalfont UNIST}\\
{\normalfont\texttt{\{hoonably, whtjddns1234, belluno, taesik.gong\}@unist.ac.kr}}}
\iclrfinalcopy %

\begin{document}
\maketitle
\fancyhead{} %
\renewcommand{\headrulewidth}{0pt}
\vspace{-0.12in} %

\begin{abstract}
Long-form reasoning makes inference expensive, and speculative decoding mitigates this cost by verifying multiple draft tokens in parallel.
Its speedup, however, can fade as context grows and draft acceptance declines.
We focus on attention-mass dilution: as softmax normalizes over more visible Keys, the mass concentrated on the highest-scoring Keys can decrease.
We introduce \method{}, a training-free method that counteracts this effect through cardinality-aware Query scaling, without the computational overhead of online adaptation.
Under explicit assumptions, we derive an exact top-$k$ mass correction and deploy a closed-form fixed-slope approximation.
Across AIME-26, GPQA-Diamond, and LongGenBench Diary, \method{} achieves $2.59$--$3.19\times$ geometric-mean end-to-end speedups over target-only autoregressive decoding when applied to DFlash, PARD, and EAGLE 3.1.
With DFlash, it improves decoding speed and outperforms full-parameter and LoRA-based online adaptation in end-to-end speedup, while matching the unmodified drafter's reported peak allocated GPU memory.
\end{abstract}

\begin{figure}[!t]
 \centering
 \begin{subfigure}[t]{0.31\linewidth}
  \centering
  \includegraphics[width=\linewidth]{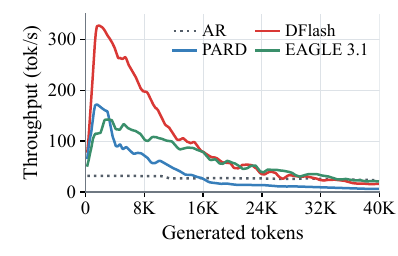}
  \caption{Decode throughput}
 \end{subfigure}\hspace{0.04\linewidth}%
 \begin{subfigure}[t]{0.31\linewidth}
  \centering
  \includegraphics[width=\linewidth]{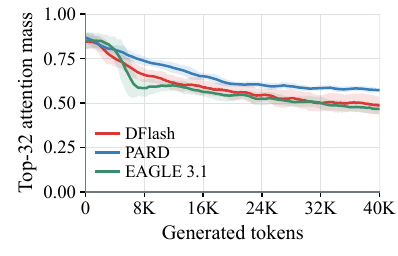}
  \caption{Top-32 attention mass}
 \end{subfigure}
 \caption{Long-generation behavior of Qwen3-8B on LongGenBench Diary.
 (a) Decode throughput of different methods in tokens per second.
 (b) Top-32 attention mass of different methods.}
 \label{fig:cardinality-motivation}
\end{figure}

\section{Introduction}
\label{sec:introduction}

Recent advances in Large Language Models (LLMs) have expanded their use from short question answering to long-form writing, long-horizon reasoning, agentic workflows, and multi-turn dialogue~\citep{longwriter,longreason,longcot,agentlongbench,mtbench101}.
These workloads either produce long outputs within a single request or repeatedly append new interactions to an already growing history.
Autoregressive decoding remains slow in this setting because it generates only one token per sequential decoding step.
Speculative decoding reduces the number of sequential target-model passes by using a lightweight draft model to propose several tokens for parallel verification~\citep{leviathan2023speculative,chen2023speculative}.
A successful verification can accept multiple draft tokens at once rather than producing a single token as in autoregressive decoding.
The resulting speedup depends in part on the mean accepted length (MAL), reported here as the number of output tokens emitted per target verification.

The benefit of speculative decoding does not remain constant throughout a long generation.
MAL can decrease as generation proceeds.
This increases the number of target-model verification steps required per generated token.
DFlash's Qwen3-8B draft model is trained with a maximum sequence length of 3,072 tokens~\citep{dflash}.
Figure~\ref{fig:cardinality-motivation}(a) illustrates this behavior on LongGenBench Diary: DFlash, PARD, and EAGLE 3.1 achieve high throughput relative to AR early in generation, but their advantage declines as generation continues.
Throughout this paper, EAGLE 3.1 refers to the post-norm EAGLE-3 configuration incorporating the normalization changes introduced in Attention Drift~\citep{attentiondrift} and described in the EAGLE 3.1 release~\citep{eagle31}.

OSD and OnlineSpec use target feedback to adapt the draft model during inference and improve its agreement with the target~\citep{osd,onlinespec}.
We ask whether growing context is also accompanied by a systematic change in the draft model's attention.
Softmax attention can become more diffuse as the number of visible Keys grows~\citep{softmaxnotenough}.
Figure~\ref{fig:cardinality-motivation}(b) shows a corresponding decline in the average attention mass assigned to the 32 highest-scoring Keys across these draft models.
We call this change \emph{attention-mass dilution} and investigate whether it contributes to MAL degradation and whether directly correcting it can preserve draft acceptance. %

In this paper, we introduce \method{}, a training-free method that counteracts attention-mass dilution through cardinality-aware Query scaling.
At each draft step, \method{} scales each draft Query according to its visible-Key cardinality without changing the draft-model parameters. %
Under explicit assumptions on the attention-score partitions, we derive an exact correction for top-$k$ attention mass.
\method{} uses a closed-form fixed-slope approximation that requires only a scalar computation and Query rescaling.
The cached Keys and Values remain unchanged.

We evaluate \method{} with DFlash, PARD, and EAGLE 3.1 on long-generation workloads.
Across AIME-26, GPQA-Diamond, and LongGenBench Diary, these systems achieve geometric-mean end-to-end speedups of $2.59$--$3.19\times$ over target-only AR.
In the DFlash evaluation, \method{} improves on the unmodified drafter while matching its reported peak allocated GPU memory, and outperforms full-parameter and LoRA-based online adaptation in end-to-end speedup.

\section{Related Work}
\label{sec:related-work}

\subsection{Speculative Decoding and Online Draft Adaptation}

Speculative decoding reduces the latency of autoregressive LLM inference~\citep{leviathan2023speculative,chen2023speculative}.
A lightweight draft model proposes multiple tokens, which the target model verifies in parallel.
The resulting speedup depends in part on how many draft tokens are accepted.
Recent advances, including EAGLE-3~\citep{eagle3} and DFlash~\citep{dflash}, improve draft quality by using target-model features to guide token prediction.
However, draft models trained on short sequences may not generalize well to longer contexts.
LongSpec addresses this mismatch in part by introducing novel position indices for short-sequence draft training~\citep{longspec}.
Test-Time Speculation (TTS) reports that acceptance in EAGLE-3 and DFlash declines as generation proceeds, reducing their decoding speedup~\citep{tts}.

Online draft adaptation addresses inference-time mismatch by learning from target feedback collected during serving.
Online Speculative Decoding (OSD) continually distills draft models on observed queries to adapt to the serving distribution~\citep{osd}.
OnlineSpec formulates draft--target interaction as online learning and develops algorithms for updating draft models from verification feedback~\citep{onlinespec}.
TTS adapts the draft model during generation using target supervision from successive verification rounds~\citep{tts}.
It uses strided updates and asynchronous execution to reduce the impact of training on decoding latency.
Our work also addresses declining draft acceptance during long generations, but focuses on the dilution of attention mass as the number of attended Keys grows.
Rather than updating the draft model weights, \method{} mitigates this effect by scaling Queries according to the current Key count.

\subsection{Attention Normalization and Length Generalization}

Attention normalization can affect how models generalize to longer contexts.
LM-Infinite shows that increasing the number of attended tokens can push attention entropy beyond its training range when attention logits remain bounded~\citep{lminfinite}.
\citet{softmaxnotenough} establish conditions under which softmax attention disperses as the number of attended tokens grows.
They propose adjusting the inference-time temperature based on attention entropy to mitigate this effect.
\citet{normalizationlimits} analyze how softmax scaling and the geometry of token representations constrain selective attention.
Scalable-Softmax mitigates attention dilution by scaling Queries logarithmically with the number of attended tokens~\citep{scalablesoftmax}.
It considers both learning the scaling parameters during pretraining and replacing softmax after pretraining.

Unlike these studies, our work addresses declining draft acceptance in speculative decoding during long generations. To that end, we analyze attention normalization to distinguish the effect of Key count from changes in the score contrast between the top-$k$ Keys and the remaining Keys.
This analysis yields a Query scaling rule that preserves a reference top-$k$ attention mass under stated assumptions.
\method{} uses a lightweight approximation of this rule, which we evaluate in terms of draft acceptance and end-to-end decoding latency.

\section{Cardinality-Aware Query Scaling}
\label{sec:method}

\method{} scales draft queries according to the number of visible keys. Our primary evaluations use a closed-form fixed-slope approximation, derived below by separating the effect of key count from attention-score contrast.

\subsection{Motivation: Attention Dilution under Growing Key Counts}
\label{sec:formulation}

Full-context drafting exposes queries to more keys as generation proceeds; Figure~\ref{fig:cardinality-motivation} shows declining top-32 attention mass for Qwen3-8B. Adding lower-scoring keys enlarges the softmax denominator and reduces an existing top-$k$ group's mass even when all existing scores stay fixed.

For one attention head of dimension $d$, fix a query $q_t$ and keys $k_{t,i}$ after normalization and positional transformation. The unscaled scores and attention weights at query scale $\beta>0$ are
\begin{equation}
 z_{t,i}=\frac{q_t^\top k_{t,i}}{\sqrt d},
 \qquad a_{t,i}(\beta)=
 \frac{e^{\beta z_{t,i}}}{\sum_{j=1}^{N_t}e^{\beta z_{t,j}}},
 \label{eq:attention}
\end{equation}
where $\beta=1$ is the checkpoint scale. For a fixed positive integer $k$ and $N_t>k$ keys visible under the actual mask, let $S_t$ be the top-$k$ score set, $B_t$ its complement, and $C_t=N_t-k=|B_t|$. To separate score contrast from group size, define
\begin{equation}
 F_t(\beta)=
 \log\left(\frac{1}{k}\sum_{i\in S_t}e^{\beta z_{t,i}}\right)
 -\log\left(\frac{1}{C_t}\sum_{j\in B_t}e^{\beta z_{t,j}}\right),
 \label{eq:method-current-contrast}
\end{equation}
The top-$k$ mass $p_t(\beta)=\sum_{i\in S_t}a_{t,i}(\beta)$ then obeys the exact identity
\begin{equation}
 \logit p_t(\beta)=F_t(\beta)-\log\frac{C_t}{k}.
 \label{eq:method-partition-identity}
\end{equation}
Here $F_t$ measures the score contrast normalized by each partition's size, while $\log(C_t/k)$ gives the explicit competitor-count term. At fixed contrast and $k$, doubling $C_t$ halves the attention odds and reduces top-$k$ mass.

During generation, content, positions, and top-$k$ membership also change, so this identity does not isolate cardinality as the cause of Figure~\ref{fig:cardinality-motivation}'s trend. Top-$k$ membership alone does not establish task relevance. For fixed scores, positive scaling preserves their ordering and, under strict top-$k$ separation, increases $F_t$ (Appendix~\ref{app:derivations}), motivating a correction for the count term.

\subsection{Cardinality-Aware Query Scaling}
\label{sec:rule}
\label{sec:reference}

Choose a fixed reference count $N_0>k$ with $C_{\mathrm{ref}}=N_0-k$. For the derivation and numerical-inverse variant, define the reference contrast $F_{\mathrm{ref}}(\beta)=\mathbb{E}_{x\in\mathcal C_{\mathrm{ref}}}[F_x(\beta)]$, where $\mathcal C_{\mathrm{ref}}$ is a set of reference queries disjoint from evaluation. The reference mass is defined by $\logit\rho_{\mathrm{ref}}=F_{\mathrm{ref}}(1)-\log(C_{\mathrm{ref}}/k)$. For $C_t\ge C_{\mathrm{ref}}$ and inactive upper clipping, reference matching requires
\begin{equation}
 F_{\mathrm{ref}}(\beta_t)-F_{\mathrm{ref}}(1)
 =\log\frac{C_t}{C_{\mathrm{ref}}}.
 \label{eq:contrast-correction}
\end{equation}
Thus, increased reference contrast offsets the log-odds loss from more competitors.

A first-order Taylor expansion around $\beta=1$ gives $F_{\mathrm{ref}}(\beta)\approx F_{\mathrm{ref}}(1)+\Delta_{\mathrm{ref}}(\beta-1)$, with $\Delta_{\mathrm{ref}}=F'_{\mathrm{ref}}(1)>0$. Our deployed fixed-slope rule replaces this derivative by a fixed $s>0$ and caps the scale at $\beta_{\max}\ge1$:
\begin{equation}
 r_t=\max\left\{\frac{N_t-k}{N_0-k},1\right\},
 \qquad
 \widehat\beta_t=\min\left\{\beta_{\max},1+\frac{\log r_t}{s}\right\}.
 \label{eq:fixed-slope-rule}
\end{equation}
We apply $q_t\leftarrow\widehat\beta_t q_t$ using $s=5$, $k=32$, and $\beta_{\max}=e^{0.5}$, without constructing an empirical reference curve. The slope is a fixed setting rather than a measured derivative. Likewise, $N_0$ is a fixed experimental proxy; the values and their interpretation are given in Appendix~\ref{app:reference-selection}. Below the reference count, $\widehat\beta_t=1$.

For $C_t\ge C_{\mathrm{ref}}$, the reference inverse preserves $\rho_{\mathrm{ref}}$ when current and reference contrasts match throughout the application range and upper clipping is inactive (Appendix~\ref{app:inverse-rule}). For the fixed-slope rule, mass error also depends on the slope choice and the Taylor approximation error (Appendix~\ref{app:approximation-scope}). This attention-mass analysis does not guarantee draft acceptance or latency.

\subsection{Inference Integration}
\label{sec:inference}

We integrate \method{} into DFlash~\citep{dflash} by rescaling queries after Q-normalization, sharing the scale across all five draft layers in our primary evaluation. The scaling affects query scores against both context and current-block keys. The target model and verification procedure are unchanged, as is the construction of persistent context keys and values from target-model features. Scale changes therefore require no rebuilding of cached context keys and values. Appendix~\ref{app:query-rescaling} specifies the BF16/FP32 arithmetic.

At each drafting step, $N_t$ includes visible persistent-context and current-draft-block keys, excluding masked and padded positions. Our full-attention implementation uses persistent-context length plus draft-block size and retains the count across turns when context is retained. With a fixed sliding window, the count and scale saturate. The parameter $k$ defines the analytical partition and does not restrict which keys receive attention. The deployed rule requires no online top-$k$ selection or reference-curve evaluation: it computes a scalar scale from the count and rescales the queries. Both operations are included in E2E latency, with no additional model forward pass, backward pass, or optimizer state.

\FloatBarrier
\section{Experiments}
\label{sec:experiments}

We evaluate how \method{} affects decoding latency, memory use, and draft acceptance across models and generation lengths.

\paragraph{Evaluation setup.}
Our primary experiments pair Qwen3-4B, Qwen3-8B, and Llama-3.1-8B with their corresponding DFlash draft models~\citep{dflash}.
We compare the unmodified draft model (Frozen), Full Adaptation, LoRA, and \method{}, using target-only autoregressive decoding (AR) as the latency reference.
Full Adaptation updates all parameters of the draft model using target feedback during generation.
LoRA~\citep{hu2022lora} updates only the low-rank adapter parameters, with rank 16 unless otherwise specified.
Both adaptation baselines use a forward-KL distillation objective to align the draft model's token distributions with those of the target~\citep{osd,distillspec}.
They use target distributions obtained during verification and schedule periodic updates during generation, a setting similar to that of Test-Time Speculation (TTS)~\citep{tts}.
\method{} applies Equation~\ref{eq:fixed-slope-rule} without updating model parameters.
We measure end-to-end (E2E) speedup, peak allocated memory, online update cost, and mean accepted length (MAL), defined as emitted tokens per target verification.
Hardware, execution backends, and measurement protocols are specified in Appendix~\ref{app:reproducibility}.

\subsection{End-to-End Efficiency}
\label{sec:main-results}

To measure the net benefit of calibration and online adaptation, we compare total generation time, including their inference-time costs.
Table~\ref{tab:main-results} covers 30 AIME-26 questions from MathArena~\citep{matharena2026}, all 198 GPQA-Diamond questions~\citep{gpqa}, and 100 Daily Diary prompts from LongGenBench's long split~\citep{longgenbench}, abbreviated as LGB(Diary).
Speedup is total AR time divided by total method time on the same prompts.
Overall is the geometric mean across benchmarks.

\newcommand{\DiaryInterimScope}{}

\begin{table}[H]
 \centering
 \caption{End-to-end speedup and adaptation overhead on long-generation workloads. Speedups are relative to target-only AR; Overall is the geometric mean across the three benchmarks. Peak memory and update cost are measured on AIME-26 and GPQA-Diamond only.}
 \label{tab:main-results}
 \resizebox{\linewidth}{!}{%
 \begin{tabular}{@{}llccccccc}
  \toprule
  \multirow{2}{*}{Model} & \multirow{2}{*}{Method}
  & \multirow{2}{*}{\shortstack{Trainable\\Params $\downarrow$}}
  & \multirow{2}{*}{\shortstack{Peak\\Mem. $\downarrow$}}
  & \multirow{2}{*}{\shortstack{Online Update\\Cost (ms/update) $\downarrow$}}
  & \multicolumn{4}{c}{Speedup $\uparrow$} \\
  \cmidrule(lr){6-9}
  & & & & & AIME-26 & GPQA-D & LGB(Diary) & Overall \\
  \midrule
  \multirow{7}{*}{Qwen3-4B}
  & Target-only AR & $0$ & $12.54\,\mathrm{GiB}$ & $0$ & $1.000\times$ & $1.000\times$ & $1.000\times$ & $1.000\times$ \\
  & DFlash (Frozen) & $0$ & $16.68\,\mathrm{GiB}$ & $0$ & $1.267\times$ & $1.598\times$ & $1.637\times$ & $1.491\times$ \\
  & $+$ Full Adaptation & $537.43\mathrm{M}$ & $24.46\,\mathrm{GiB}$ & $103.14$ & $1.368\times$ & $1.981\times$ & $2.478\times$ & $1.886\times$ \\
  & $+$ LoRA ($r=32$) & $9.67\mathrm{M}$ & $19.95\,\mathrm{GiB}$ & $51.26$ & $1.434\times$ & $1.952\times$ & $2.568\times$ & $1.930\times$ \\
  & $+$ LoRA ($r=16$) & $4.83\mathrm{M}$ & $19.84\,\mathrm{GiB}$ & $49.91$ & $1.589\times$ & $2.183\times$ & $2.865\times$ & $2.150\times$ \\
  & $+$ LoRA ($r=8$) & $2.42\mathrm{M}$ & $19.80\,\mathrm{GiB}$ & $50.13$ & $1.667\times$ & $2.267\times$ & $3.050\times$ & $2.259\times$ \\
  \rowcolor{oursrow}
  \cellcolor{white} & $+$ \method{} & $0$ & $16.68\,\mathrm{GiB}$ & $0$ & $\mathbf{1.973\times}$ & $\mathbf{2.666\times}$ & $\mathbf{3.316\times}$ & $\mathbf{2.593\times}$ \\
  \midrule
  \multirow{7}{*}{Qwen3-8B}
  & Target-only AR & $0$ & $20.26\,\mathrm{GiB}$ & $0$ & $1.000\times$ & $1.000\times$ & $1.000\times$ & $1.000\times$ \\
  & DFlash (Frozen) & $0$ & $26.87\,\mathrm{GiB}$ & $0$ & $1.319\times$ & $1.392\times$ & $1.798\times$ & $1.489\times$ \\
  & $+$ Full Adaptation & $1.049\mathrm{B}$ & $41.29\,\mathrm{GiB}$ & $199.21$ & $1.148\times$ & $1.317\times$ & $2.112\times$ & $1.472\times$ \\
  & $+$ LoRA ($r=32$) & $12.91\mathrm{M}$ & $31.95\,\mathrm{GiB}$ & $61.65$ & $1.532\times$ & $1.815\times$ & $2.993\times$ & $2.027\times$ \\
  & $+$ LoRA ($r=16$) & $6.46\mathrm{M}$ & $31.80\,\mathrm{GiB}$ & $61.76$ & $1.613\times$ & $1.887\times$ & $3.125\times$ & $2.119\times$ \\
  & $+$ LoRA ($r=8$) & $3.23\mathrm{M}$ & $31.63\,\mathrm{GiB}$ & $61.76$ & $1.643\times$ & $1.884\times$ & $3.184\times$ & $2.144\times$ \\
  \rowcolor{oursrow}
  \cellcolor{white} & $+$ \method{} & $0$ & $26.87\,\mathrm{GiB}$ & $0$ & $\mathbf{2.130\times}$ & $\mathbf{2.316\times}$ & $\mathbf{3.882\times}$ & $\mathbf{2.675\times}$ \\
  \midrule
  \multirow{7}{*}{Llama-3.1-8B}
  & Target-only AR & $0$ & $19.16\,\mathrm{GiB}$ & $0$ & $1.000\times$ & $1.000\times$ & $1.000\times$ & $1.000\times$ \\
  & DFlash (Frozen) & $0$ & $25.62\,\mathrm{GiB}$ & $0$ & $2.170\times$ & $2.045\times$ & $3.183\times$ & $2.417\times$ \\
  & $+$ Full Adaptation & $1.049\mathrm{B}$ & $40.15\,\mathrm{GiB}$ & $246.19$ & $1.705\times$ & $1.596\times$ & $2.211\times$ & $1.819\times$ \\
  & $+$ LoRA ($r=32$) & $12.91\mathrm{M}$ & $30.61\,\mathrm{GiB}$ & $102.06$ & $2.044\times$ & $1.998\times$ & $2.675\times$ & $2.219\times$ \\
  & $+$ LoRA ($r=16$) & $6.46\mathrm{M}$ & $30.59\,\mathrm{GiB}$ & $101.62$ & $2.147\times$ & $2.067\times$ & $2.809\times$ & $2.319\times$ \\
  & $+$ LoRA ($r=8$) & $3.23\mathrm{M}$ & $30.22\,\mathrm{GiB}$ & $101.26$ & $2.175\times$ & $2.100\times$ & $2.854\times$ & $2.353\times$ \\
  \rowcolor{oursrow}
  \cellcolor{white} & $+$ \method{} & $0$ & $25.62\,\mathrm{GiB}$ & $0$ & $\mathbf{2.714\times}$ & $\mathbf{2.593\times}$ & $\mathbf{3.710\times}$ & $\mathbf{2.967\times}$ \\
  \bottomrule
 \end{tabular}%
 }
\end{table}

\method{} achieves the highest E2E speedup on all nine model--benchmark pairs in Table~\ref{tab:main-results}.
Overall speedups are $2.593\times$, $2.675\times$, and $2.967\times$ for Qwen3-4B, Qwen3-8B, and Llama-3.1-8B, respectively.
Full Adaptation improves over Frozen on Qwen3-4B.
For Qwen3-8B and Llama-3.1-8B, however, it takes longer on both AIME-26 and GPQA-Diamond.
All tested LoRA ranks trail \method{} on every model--benchmark pair.

Query scaling also avoids the memory needed to train the draft model.
\method{} matches Frozen's reported peak memory, whereas Full Adaptation adds $7.78$--$14.53$ GiB across the three models.
LoRA reduces this overhead but still incurs additional memory and update costs.
Query scaling is included in the E2E latency and peak-memory measurements.
\method{} requires no optimizer updates.

To examine the training work incurred within a request, Figure~\ref{fig:request-time-breakdown} overlays recorded GPU training durations on mean AIME-26 E2E times, using rank 8 for LoRA.
Full Adaptation spends substantially more GPU time on training than LoRA, while \method{} completes requests fastest on all three models without optimizer updates.
Training can overlap decoding, so the hatched durations measure training work rather than a separate, additive portion of request latency (Appendix~\ref{app:request-time-breakdown}).

\begin{figure}[H]
 \centering
 \includegraphics[width=0.96\linewidth]{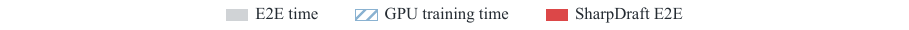}\par
 \captionsetup{font=small,labelfont=normalfont,skip=3pt}
 \captionsetup[subfigure]{font=small,skip=0.5pt}
 \begin{subfigure}[t]{0.42\linewidth}
  \centering
  \includegraphics[width=0.98\linewidth]{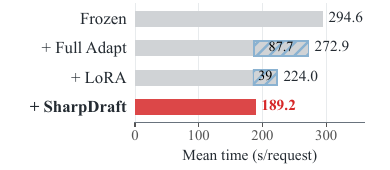}
  \caption{Qwen3-4B}
 \end{subfigure}\hfill
 \begin{subfigure}[t]{0.28\linewidth}
  \centering
  \includegraphics[width=0.98\linewidth]{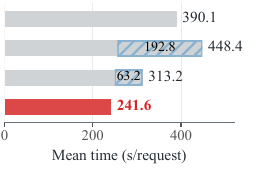}
  \caption{Qwen3-8B}
 \end{subfigure}\hfill
 \begin{subfigure}[t]{0.28\linewidth}
  \centering
  \includegraphics[width=0.98\linewidth]{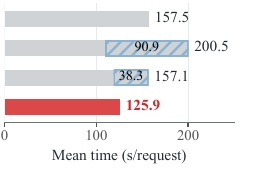}
  \caption{Llama-3.1-8B}
 \end{subfigure}
 \caption{Mean end-to-end latency and recorded GPU training time per AIME-26 request.}
 \label{fig:request-time-breakdown}
\end{figure}

\subsection{Generalization Across Draft Model Families}
\label{sec:pard-extension}

To test whether the correction is useful beyond DFlash, we use Qwen3-8B and apply Query scaling to parallel drafting with PARD~\citep{pard} and autoregressive drafting with EAGLE 3.1~\citep{eagle31,attentiondrift}.
The draft model weights, proposal configuration, and verification procedure remain unchanged within each family.
Table~\ref{tab:pard-results} reports MAL and AR-relative speedup.
Reference and clipping settings for each draft model are given in Appendix~\ref{app:cross-drafter}.

For EAGLE 3.1, \method{} raises MAL from $3.152$ to $3.302$ on AIME-26, $3.659$ to $3.904$ on GPQA-D, and $5.668$ to $5.968$ on Diary.
Whether the MAL gains reduce latency also depends on drafting and verification costs.

\begin{table}[!htbp]
\centering
\caption{MAL and end-to-end speedup across draft model families with Qwen3-8B. Speedups are relative to target-only AR; Overall is their geometric mean across the three benchmarks.}
\label{tab:pard-results}
\fontsize{8}{9.5}\selectfont
\setlength{\tabcolsep}{4pt}
\renewcommand{\arraystretch}{1.08}
\begin{tabular}{@{}lccccccc}
\toprule
\multirow{2.5}{*}{Method}
 & \multicolumn{3}{c}{MAL $\uparrow$}
 & \multicolumn{4}{c}{Speedup $\uparrow$} \\
\cmidrule(lr){2-4}\cmidrule(lr){5-8}
 & AIME-26 & GPQA-D & LGB(Diary) & AIME-26 & GPQA-D & LGB(Diary) & Overall \\
\midrule
Target-only AR & -- & -- & -- & $1.000\times$ & $1.000\times$ & $1.000\times$ & $1.000\times$ \\
PARD & $2.368$ & $2.728$ & $2.252$ & $2.691\times$ & $3.389\times$ & $2.729\times$ & $2.919\times$ \\
$+$ LoRA & $1.637$ & $2.224$ & $1.390$ & $1.015\times$ & $1.567\times$ & $0.923\times$ & $1.137\times$ \\
\cellcolor{oursrow}$+$ \method{} & \cellcolor{oursrow}$\mathbf{2.441}$ & \cellcolor{oursrow}$\mathbf{2.906}$ & \cellcolor{oursrow}$\mathbf{2.424}$ & \cellcolor{oursrow}$\mathbf{2.773\times}$ & \cellcolor{oursrow}$\mathbf{3.939\times}$ & \cellcolor{oursrow}$\mathbf{2.975\times}$ & \cellcolor{oursrow}$\mathbf{3.191\times}$ \\
EAGLE 3.1 & $3.152$ & $3.659$ & $5.668$ & $1.922\times$ & $2.529\times$ & $3.515\times$ & $2.575\times$ \\
$+$ LoRA & $3.086$ & $3.812$ & $4.251$ & $1.725\times$ & $2.168\times$ & $2.549\times$ & $2.120\times$ \\
\cellcolor{oursrow}$+$ \method{} & \cellcolor{oursrow}$\mathbf{3.302}$ & \cellcolor{oursrow}$\mathbf{3.904}$ & \cellcolor{oursrow}$\mathbf{5.968}$ & \cellcolor{oursrow}$\mathbf{2.045\times}$ & \cellcolor{oursrow}$\mathbf{2.624\times}$ & \cellcolor{oursrow}$\mathbf{3.792\times}$ & \cellcolor{oursrow}$\mathbf{2.730\times}$ \\
\bottomrule
\end{tabular}
\end{table}

\subsection{Draft Acceptance in Multi-Turn Generation}
\label{sec:multiturn-mal}

To evaluate acceptance as conversation history grows, we extend the multi-turn setting used in MT-Bench~\citep{mtbench} and speculative-decoding benchmarks~\citep{eagle,specbench} with MATH-500 level-4/5 problems~\citep{lightman2023verify}.
Each turn introduces a new problem after EOS or 5,000 generated tokens, retaining conversation history, KV caches, and adaptation state.
Figure~\ref{fig:multiturn-mal-profiles} averages thirty sessions with matched question orders and covers 32K generated tokens for the three Table~\ref{tab:main-results} models.
To examine longer histories, we additionally evaluate GLM-4.7-Flash with a cumulative generation budget of 100K tokens, using its larger context limit.

\method{} maintains higher late-position MAL than Frozen across all four models.
Full Adaptation and LoRA can achieve higher MAL, particularly for Qwen3-4B, at the cost of online training.
Repetitive continuations are retained, so these curves measure drafting efficiency rather than reasoning quality (Appendix~\ref{app:multiturn-mal}).

\begingroup
\setlength{\intextsep}{6pt}
\begin{figure}[H]
 \centering
 \includegraphics[width=0.96\linewidth]{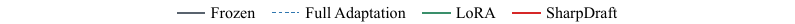}\par
 \vspace{0.1em}
 \captionsetup[subfigure]{font=footnotesize,skip=1pt}
 \input{figures/data/figure4_mal_layout.tex}
 \begin{subfigure}[t]{\FigureFourPanelAWidth}
  \centering
  \includegraphics[width=\linewidth]{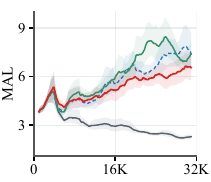}
  \caption{Qwen3-4B}
  \label{fig:mal-qwen3-4b-multiturn}
 \end{subfigure}\hspace{\FigureFourPanelGap}%
 \begin{subfigure}[t]{\FigureFourPanelBWidth}
  \centering
  \includegraphics[width=\linewidth]{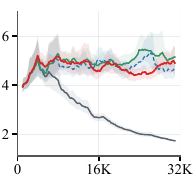}
  \caption{Qwen3-8B}
  \label{fig:mal-qwen3-8b-multiturn}
 \end{subfigure}\hspace{\FigureFourPanelGap}%
 \begin{subfigure}[t]{\FigureFourPanelCWidth}
  \centering
  \includegraphics[width=\linewidth]{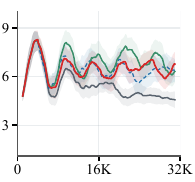}
  \caption{Llama-3.1-8B}
  \label{fig:mal-llama31-8b-multiturn}
 \end{subfigure}\hspace{\FigureFourPanelGap}%
 \begin{subfigure}[t]{\FigureFourPanelDWidth}
  \centering
  \includegraphics[width=\linewidth]{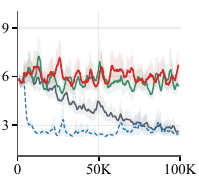}
  \caption{GLM-4.7-Flash}
  \label{fig:mal-glm47-multiturn}
 \end{subfigure}
 \caption{Draft acceptance during multi-turn MATH-500 generation, up to 32K tokens in (a)--(c) and 100K in (d). Lines show mean session MAL; bands indicate pointwise 95\% bootstrap confidence intervals.}
 \label{fig:multiturn-mal-profiles}
\end{figure}
\endgroup

\FloatBarrier

\raggedbottom
\section{Analysis and Ablation Studies}
\label{sec:analysis}

\subsection{Dynamic versus Fixed Scaling}
\label{sec:fixed-ablation}

To isolate the benefit of context-dependent scaling, we compare \method{} with fixed Query scales on Qwen3-8B/DFlash.
The comparison includes Frozen ($\beta=1$) and five constants, $\beta\in\{e^{0.1},e^{0.2},e^{0.3},e^{0.4},e^{0.5}\}$.
All seven settings produce identical output tokens and termination on the 30 AIME-26 problems.

\newcommand{\AblationPaired}{30}
\newcommand{\AblationNativeMAL}{3.182}
\newcommand{\AblationDynamicMAL}{4.063}
\newcommand{\AblationBestFixedMAL}{4.023}
\newcommand{\AblationBestFixedBeta}{e^{0.3}}
\newcommand{\AblationDynamicDelta}{+0.040}
\newcommand{\AblationDynamicPercent}{0.99}
\newcommand{\AblationDynamicWins}{26}
\newcommand{\AblationShortN}{9}
\newcommand{\AblationMediumN}{8}
\newcommand{\AblationLongN}{13}
\newcommand{\AblationPositionEarlyN}{30}
\newcommand{\AblationPositionMiddleN}{21}
\newcommand{\AblationPositionLateN}{13}
\newcommand{\AblationShortFixedDelta}{+0.007}
\newcommand{\AblationMediumFixedDelta}{+0.023}
\newcommand{\AblationLongFixedDelta}{+0.018}

Table~\ref{tab:fixed-dynamic-length} reports whole-request MAL, weighting problems equally and grouping them by final output length.
The best fixed scale is the one with the highest mean MAL among the five tested values, selected separately overall and within each length group (Appendix~\ref{app:fixed-dynamic-live}).

\begingroup
\setlength{\intextsep}{6pt}
\begin{table}[H]
 \centering
 \fontsize{8}{9.5}\selectfont
 \caption{Fixed versus dynamic Query scaling on Qwen3-8B/DFlash using 30 output-matched AIME-26 problems. Values are mean whole-request MAL, grouped by final output length. Parentheses show gains over Frozen.}
 \label{tab:fixed-dynamic-length}
 \setlength{\tabcolsep}{4pt}
 \renewcommand{\arraystretch}{1.12}
 \begin{tabular}{lcccc}
 \toprule
 \multirow[c]{2.5}{*}{Method} & \multicolumn{3}{c}{Final output length} & All problems \\
 \cmidrule(lr){2-4}
  & 0--10K & 10--20K & 20--32K & Overall MAL $\uparrow$ \\
 \midrule
 Frozen ($\beta=1$) & 4.092 & 3.113 & 2.596 & 3.182 \\
 Fixed $\beta=e^{0.1}$ & 4.273 {\scriptsize\textcolor{gray}{(+0.181)}} & 3.524 {\scriptsize\textcolor{gray}{(+0.411)}} & 3.318 {\scriptsize\textcolor{gray}{(+0.723)}} & 3.660 {\scriptsize\textcolor{gray}{(+0.477)}} \\
 Fixed $\beta=e^{0.2}$ & 4.305 {\scriptsize\textcolor{gray}{(+0.214)}} & 3.726 {\scriptsize\textcolor{gray}{(+0.613)}} & 3.785 {\scriptsize\textcolor{gray}{(+1.189)}} & 3.925 {\scriptsize\textcolor{gray}{(+0.743)}} \\
 Fixed $\beta=e^{0.3}$ & 4.284 {\scriptsize\textcolor{gray}{(+0.193)}} & 3.784 {\scriptsize\textcolor{gray}{(+0.671)}} & 3.989 {\scriptsize\textcolor{gray}{(+1.393)}} & 4.023 {\scriptsize\textcolor{gray}{(+0.841)}} \\
 Fixed $\beta=e^{0.4}$ & 4.235 {\scriptsize\textcolor{gray}{(+0.143)}} & 3.756 {\scriptsize\textcolor{gray}{(+0.643)}} & 4.028 {\scriptsize\textcolor{gray}{(+1.433)}} & 4.018 {\scriptsize\textcolor{gray}{(+0.835)}} \\
 Fixed $\beta=e^{0.5}$ & 4.166 {\scriptsize\textcolor{gray}{(+0.074)}} & 3.702 {\scriptsize\textcolor{gray}{(+0.588)}} & 4.005 {\scriptsize\textcolor{gray}{(+1.410)}} & 3.973 {\scriptsize\textcolor{gray}{(+0.790)}} \\
 \rowcolor{oursrow}
 \method{} (dynamic) & \textbf{4.313} {\scriptsize\textcolor{gray}{(+0.221)}} & \textbf{3.808} {\scriptsize\textcolor{gray}{(+0.694)}} & \textbf{4.047} {\scriptsize\textcolor{gray}{(+1.451)}} & \textbf{4.063} {\scriptsize\textcolor{gray}{(+0.880)}} \\
 \bottomrule
 \end{tabular}
\end{table}

\endgroup

Fixed scaling captures most of the gain over Frozen.
The best fixed scale ($\beta=\AblationBestFixedBeta$) reaches a mean MAL of \AblationBestFixedMAL{}, compared with \AblationNativeMAL{} for Frozen and \AblationDynamicMAL{} for \method{}.
The additional gain from dynamic scaling is modest ($\AblationDynamicPercent\%$), but it improves request MAL on \AblationDynamicWins{} of \AblationPaired{} problems and has the highest mean in each length group.
The preferred fixed scale also changes with output length, from $e^{0.2}$ for short outputs to $e^{0.4}$ for long outputs.
A single fixed setting must therefore compromise between these length regimes.

\newcommand{\AimeAttentionProblems}{30}
\newcommand{\AimeAttentionStates}{416}
\newcommand{\AimeAttentionMeasurements}{2,912}
\newcommand{\AimeAttentionChecks}{420}
\newcommand{\AimeAttentionSupport}{30/29/26/15/4}
\newcommand{\AimeAttentionLateN}{4}
\newcommand{\AimeAttentionDynamicEarlyMin}{0.774}
\newcommand{\AimeAttentionDynamicEarlyMax}{0.778}
\newcommand{\AimeAttentionDynamicLate}{0.727}
\newcommand{\AimeAttentionFrozenEarly}{0.774}
\newcommand{\AimeAttentionFrozenLate}{0.414}
\newcommand{\AimeAttentionFixedEarlyMin}{0.825}
\newcommand{\AimeAttentionFixedEarlyMax}{0.953}
\newcommand{\AimeAttentionFixedLateMin}{0.490}
\newcommand{\AimeAttentionFixedLateMax}{0.760}

\newcommand{\AimeCumulativeSupport}{30/26/15/9/4}
\newcommand{\AimeCumulativeLateN}{4}

To see whether this trade-off also occurs within a generation, we track attention concentration and cumulative acceptance on the same cohort.
Fixed scales raise top-32 attention mass on matched prefixes, but the mass still declines as context grows.
\method{} stays near its short-context level through 16K before declining at 32K (Figure~\ref{fig:dynamic-ablation-mass}).
The live acceptance profiles show a corresponding shift toward stronger fixed scales as generation proceeds, whereas \method{} has the highest mean cumulative MAL from 4K onward (Figure~\ref{fig:dynamic-ablation-cumulative-mal}).
Later positions average only the problems that reach them, as detailed in Appendices~\ref{app:attention-sweep} and~\ref{app:cumulative-mal}.

\begingroup
\setlength{\intextsep}{6pt}
\begin{figure}[H]
 \centering
 \includegraphics[width=0.66\linewidth]{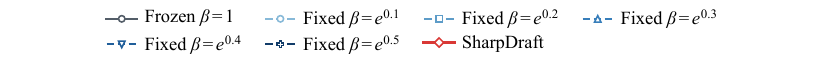}\par
 \vspace{-\parskip}\nointerlineskip
 \begin{subfigure}[t]{0.31\linewidth}
  \centering
  \includegraphics[width=\linewidth]{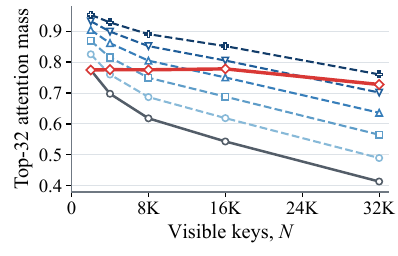}
  \caption{Attention concentration}
  \label{fig:dynamic-ablation-mass}
 \end{subfigure}\hspace{0.04\linewidth}%
 \begin{subfigure}[t]{0.31\linewidth}
  \centering
  \includegraphics[width=\linewidth]{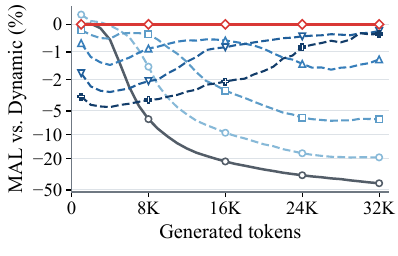}
  \caption{MAL relative to Dynamic}
  \label{fig:dynamic-ablation-cumulative-mal}
 \end{subfigure}
 \caption{Fixed and dynamic scaling during AIME-26 generation with Qwen3-8B. (a) Top-32 attention mass on matched prefixes. (b) Mean cumulative MAL relative to \method{}.}
 \label{fig:dynamic-ablation}
\end{figure}
\endgroup

\subsection{Fixed-Slope Approximation}
\label{sec:taylor-center}
\label{sec:effective-slope}

We examine whether more accurate attention-mass restoration improves acceptance enough to justify the cost of state-dependent calibration.
Although \method{} fixes the shared slope at $s=5$, its Query scale still varies with visible-Key count through Equation~\ref{eq:fixed-slope-rule}.
The approximation avoids measuring current attention statistics and computes the scale from cardinality alone.

For a more accurate correction, Layerwise Taylor evaluates the current contrast and its derivative separately in each draft layer, re-centering the Taylor expansion near the previous applied scale.
This requires additional Query--Key score computation and top-32 partitioning in all five layers at each post-reference drafting step.
We compare this variant with Frozen and \method{} on AIME-26 using Qwen3-8B/DFlash with native vLLM.
Appendix~\ref{app:layerwise-taylor} gives the update rule and measurement details.

Layerwise Taylor restores attention mass more accurately and yields a modest MAL gain.
Relative to \method{}, Layerwise Taylor reduces mean absolute top-32 mass error from $6.163$ to $2.388$ percentage points and raises whole-request MAL from $4.545$ to $4.667$ ($2.7\%$), averaging requests equally (Figures~\ref{fig:layerwise-full-mass} and~\ref{fig:layerwise-full-mal}).
Most of the additional MAL gain comes from capped or repetitive outputs.
These results are consistent with a benefit from mitigating dilution, but do not isolate its causal effect because layerwise correction, the scale cap ($2$ rather than $e^{0.5}$), and generation paths also differ.

Mean per-request E2E time is $89.70$ seconds for Frozen, $59.24$ for \method{}, and $77.90$ for Layerwise Taylor, with Taylor taking $31.5\%$ longer than \method{}.
Because output lengths and GPU assignments differ in this full-cohort comparison, we also assess runtime on requests with identical output tokens and termination for \method{} and Layerwise Taylor, measuring each comparison on the same GPU.
On these matched requests, Layerwise Taylor increases mean E2E time from $25.87$ to $31.28$ seconds ($20.9\%$).
These timings include calibration during generation but exclude separate attention-mass diagnostics.
For this native-vLLM implementation, the extra calibration cost outweighs the acceptance gain, favoring the simpler fixed-slope rule.

We separately test the choice of $s=5$ by sweeping $s\in\{3,4,5,6,8\}$ on AIME-26, holding $N_0$, $k$, and $\beta_{\max}$ fixed.
The default $s=5$ gives the highest observed whole-request MAL ($4.545$), while $s=4$ and $s=6$ remain within $1.6\%$ (Figure~\ref{fig:slope-sensitivity}).
A separate pooled-reference diagnostic yields effective slopes between $4.995$ and $5.112$ (Appendix~\ref{app:fixed-slope-validation}).
Together, these results support $s=5$ as a practical default that retains cardinality-dependent scaling without the cost of online contrast estimation.

\begin{figure}[!htbp]
 \centering
 \includegraphics[width=0.96\linewidth]{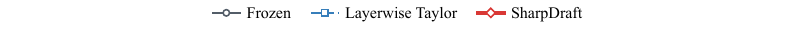}\par
 \vspace{-\parskip}\nointerlineskip
 \begin{subfigure}[t]{0.30\linewidth}
  \centering
  \includegraphics[width=\linewidth]{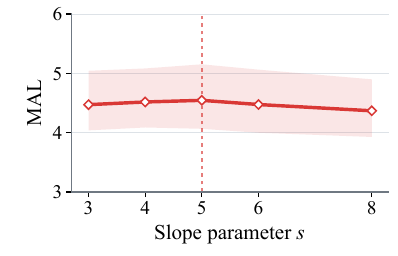}
  \caption{Slope sensitivity}
  \label{fig:slope-sensitivity}
 \end{subfigure}\hspace{0.04\linewidth}%
 \begin{subfigure}[t]{0.30\linewidth}
  \centering
  \includegraphics[width=\linewidth]{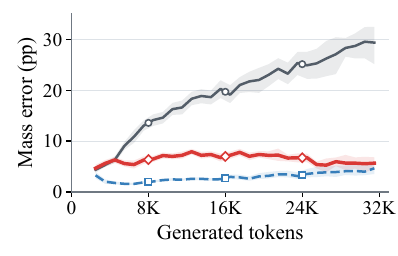}
  \caption{Top-32 mass error}
  \label{fig:layerwise-full-mass}
 \end{subfigure}\hspace{0.04\linewidth}%
 \begin{subfigure}[t]{0.30\linewidth}
  \centering
  \includegraphics[width=\linewidth]{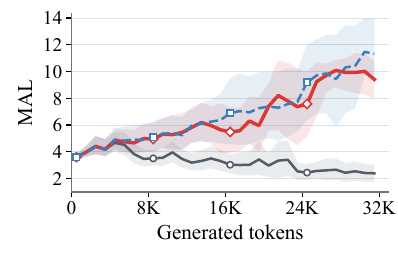}
  \caption{MAL}
  \label{fig:layerwise-full-mal}
 \end{subfigure}
 \caption{Fixed-slope sensitivity and layerwise calibration on Qwen3-8B/AIME-26. (a) MAL across slope settings. (b,c) Top-32 mass error and MAL. Shading shows pointwise 95\% bootstrap confidence intervals.}
 \label{fig:layerwise-full-profiles}
\end{figure}

\subsection{Sensitivity to the Reference Cardinality}
\label{sec:reference-sensitivity}

We vary the reference count $N_0$ to check whether the same setting remains useful across models and datasets.
Scaling begins when the visible-Key count exceeds $N_0$, so a larger reference shifts its onset to a longer context.
Figure~\ref{fig:mechanism} compares $N_0\in\{1\mathrm{K},2\mathrm{K},4\mathrm{K},8\mathrm{K}\}$ for four models on AIME-26 and three datasets for Qwen3-8B, using 30 inputs per model--dataset pair.
We report the reduction in total verifier calls per generated token relative to the default reference of 2K for Qwen/Llama and 4K for GLM.
Positive values indicate fewer calls per token, while zero denotes the default reference.

\begin{figure}[!htbp]
 \centering
 \begin{subfigure}[t]{0.31\linewidth}
  \centering
  \includegraphics[width=\linewidth]{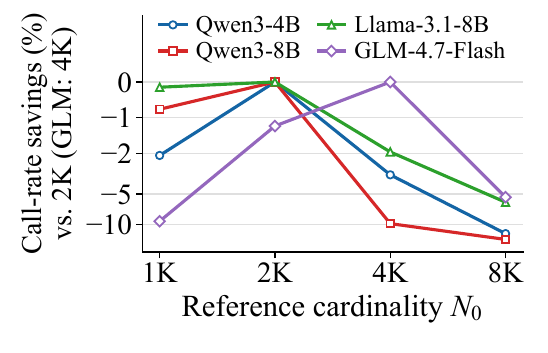}
  \caption{Across models}
  \label{fig:reference-models}
 \end{subfigure}\hspace{0.04\linewidth}%
 \begin{subfigure}[t]{0.31\linewidth}
  \centering
  \includegraphics[width=\linewidth]{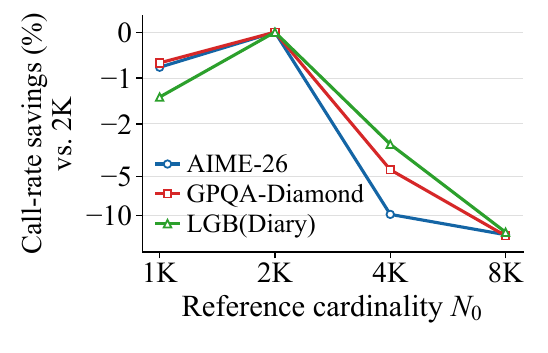}
  \caption{Across datasets}
  \label{fig:full-rollout-domain}
 \end{subfigure}
 \caption{Sensitivity to the reference cardinality across models and datasets. Verifier-call-rate savings are relative to 2K for Qwen/Llama and 4K for GLM.}
 \label{fig:mechanism}
\end{figure}

The effect of changing the reference differs across models.
On AIME-26, increasing $N_0$ from 2K to 4K raises calls per token by $3.2\%$, $9.8\%$, and $2.0\%$ for Qwen3-4B, Qwen3-8B, and Llama-3.1-8B, respectively.
GLM shows the opposite preference between these two settings, requiring $1.2\%$ more calls at 2K than at 4K.
For Qwen3-8B, the pattern is similar across AIME-26, GPQA-Diamond, and LGB(Diary).
The 1K setting stays close to 2K, while 8K requires approximately $13$--$14\%$ more calls per token.
These results support a shared reference across the tested Qwen3-8B workloads, but suggest checking it again when changing the model rather than treating 2K as universal.
Appendix~\ref{app:reference-sensitivity} provides the generation and aggregation details.
\flushbottom

\section{Discussion and Limitations}
\label{sec:limitations}

Query scaling changes how sharply the draft model attends, but cannot by itself correct every source of disagreement with the target.
For fixed Queries and Keys, positive scaling preserves score order and can reinforce an incorrect preference.
The layerwise experiment illustrates this distinction: restoring attention mass more closely adds little to MAL.
Learned updates may still be useful when the mismatch concerns content rather than concentration; combining selective online adaptation with Query scaling is a natural extension.

The theoretical guarantee is narrower than the empirical evaluation.
Our reference-inverse mass-preservation result assumes $C_t\ge C_{\mathrm{ref}}$, agreement between current and reference contrasts throughout the application range, and inactive upper clipping.
The deployed rule uses a Taylor approximation and a fixed slope, whose adequacy also depends on current/reference contrast offsets (Appendix~\ref{app:effective-slope}).
Its slope and reference count should therefore be validated for a new draft model, rather than assumed to recover a universal calibration rule.

Our end-to-end evidence is strongest for DFlash under greedy decoding; the transfer experiments show that the benefit can vary with the target--draft pair.
The multi-turn study extends acceptance measurements to 100K generated tokens, not latency measurements, and retains repetitive continuations.
Since repetitive text can be easy to draft, evaluating answer quality alongside latency is important before interpreting high long-context MAL as a practical gain.
Whether the latency gains persist under stochastic decoding, batching, or concurrent serving remains untested.

\section{Conclusion}
\label{sec:conclusion}

We presented \method{}, which scales draft Queries according to visible-Key cardinality without training the draft model.
Separating attention-score contrast from the number of competing Keys leads to a reference-based correction, implemented through a fixed-slope approximation derived from a Taylor expansion.
On the evaluated DFlash workloads, this correction yields higher end-to-end speedup than Frozen and online adaptation while matching Frozen's peak allocated memory.
The ablations show a modest but consistent benefit from varying the scale with context, while more precise mass restoration adds little further acceptance.

\clearpage %

\bibliography{iclr2027_conference}
\bibliographystyle{iclr2027_conference}
\appendix

\section{Theoretical Analysis}
\label{app:derivations}

\subsection{Exact Partition Identity}
For a current query define
\begin{equation}
 F_t(\beta)=
 \log\left(\frac{1}{k}\sum_{i\in S_t}e^{\beta z_{t,i}}\right)
 -\log\left(\frac{1}{C_t}\sum_{j\in B_t}e^{\beta z_{t,j}}\right).
 \label{eq:current-contrast}
\end{equation}
For $p_t(\beta)=\sum_{i\in S_t}a_{t,i}(\beta)$, direct algebra gives
\begin{equation}
 \logit p_t(\beta)=F_t(\beta)-\log\frac{C_t}{k}.
 \label{eq:partition-identity}
\end{equation}
At fixed $F_t(1)$, increasing $C_t$ decreases uncalibrated top-$k$ mass, without requiring the selected keys to be task-relevant.

\subsection{Reference Contrast and Monotonicity}
For numerical inversion, let $\mathcal C_{\mathrm{ref}}$ be a finite query set disjoint from evaluation, with each $F_x$ normalized by its own partition sizes as in Section~\ref{sec:formulation}. The frozen reference curve is
\begin{equation}
 F_{\mathrm{ref}}(\beta)=
 \mathbb{E}_{x\in\mathcal C_{\mathrm{ref}}}[F_x(\beta)].
 \label{eq:reference-contrast}
\end{equation}
For a specified reference competitor count $C_{\mathrm{ref}}=N_0-k>0$, define
\begin{equation}
 \logit\rho_{\mathrm{ref}}=
 F_{\mathrm{ref}}(1)-\log\frac{C_{\mathrm{ref}}}{k}.
 \label{eq:reference-mass}
\end{equation}
Only the reference-inverse rule requires the measured curve; the deployed fixed-slope rule does not.

Let $w_{x,i}^{S}(\beta)$ and $w_{x,j}^{B}(\beta)$ be softmax weights within each partition.
Differentiating gives
\begin{equation}
 F_{\mathrm{ref}}'(\beta)=
 \mathbb{E}_{x\in\mathcal{C}_{\mathrm{ref}}}\left[
 \sum_{i\in S_x}w_{x,i}^{S}(\beta)z_{x,i}
 -\sum_{j\in B_x}w_{x,j}^{B}(\beta)z_{x,j}\right].
 \label{eq:contrast-derivative}
\end{equation}
Under strict top-$k$ separation, every top-$k$ logit exceeds every background logit; the first weighted mean is therefore larger.
For a finite calibration set this establishes strict monotonicity of the average and the existence of an inverse on the selected interval.

\subsection{Reference Inverse and Conditional Mass Preservation}
\label{app:inverse-rule}
For $\beta_{\max}\ge1$, the complete reference-inverse rule is
\begin{align}
 r_t&=\max\left\{\frac{C_t}{C_{\mathrm{ref}}},1\right\},
 \label{eq:ratio}\\
 \beta_t&=F_{\mathrm{ref}}^{-1}\!\left(
 \min\left\{F_{\mathrm{ref}}(1)+\log r_t,\,
 F_{\mathrm{ref}}(\beta_{\max})\right\}\right).
 \label{eq:inverse-rule}
\end{align}
The inverse is evaluated numerically on $F_{\mathrm{ref}}([1,\beta_{\max}])$ using a frozen curve or lookup table. Apply $q_t\leftarrow\beta_t q_t$; the clamps give unit scale below $C_{\mathrm{ref}}$ and cap it at $\beta_{\max}$. The conditional guarantee below applies when neither clamp changes the matching solution and the stated contrast assumptions hold.

\begin{proposition}[Conditional reference-mass preservation]
\label{prop:mass}
Assume $C_t\ge C_{\mathrm{ref}}$, inactive upper clipping, and $F_t(\beta)=F_{\mathrm{ref}}(\beta)$ throughout the application range. The reference inverse in \eqref{eq:inverse-rule} then gives $p_t(\beta_t)=\rho_{\mathrm{ref}}$. If instead $|F_t(\beta_t)-F_{\mathrm{ref}}(\beta_t)|\le\varepsilon$, the same count and clipping conditions give
\begin{equation}
 |p_t(\beta_t)-\rho_{\mathrm{ref}}|\le\varepsilon/4.
 \label{eq:mass-error}
\end{equation}
\end{proposition}

\begin{proof}
Since $C_t\ge C_{\mathrm{ref}}$ and upper clipping is inactive, \eqref{eq:inverse-rule} gives $F_{\mathrm{ref}}(\beta_t)=F_{\mathrm{ref}}(1)+ \log(C_t/C_{\mathrm{ref}})$.
Using contrast equality and \eqref{eq:partition-identity},
\begin{align*}
 \logit p_t
 &=F_t(\beta_t)-\log(C_t/k)\\
 &=F_{\mathrm{ref}}(1)+\log(C_t/C_{\mathrm{ref}})-\log(C_t/k)\\
 &=F_{\mathrm{ref}}(1)-\log(C_{\mathrm{ref}}/k)
 =\logit\rho_{\mathrm{ref}}.
\end{align*}
Thus $p_t=\rho_{\mathrm{ref}}$.
With contrast error at most $\varepsilon$, log-odds error is at most $\varepsilon$.
The logistic function is $1/4$-Lipschitz, yielding \eqref{eq:mass-error}.
\end{proof}

\subsection{Training-Distribution Reference and Approximation Scope}
\label{app:approximation-scope}
When the training-query distribution is available, define $F_{\mathrm{train}}(\beta)=\mathbb E_{x\sim\mathcal D_{\mathrm{train}}}[F_x(\beta)]$ and the effective competitor count
\begin{equation}
 C_{\mathrm{train}}=\exp\!\left(
 \mathbb E_{x\sim\mathcal D_{\mathrm{train}}}[\log(N_x-k)]\right).
 \label{eq:training-reference}
\end{equation}
If the same training distribution defines $F_{\mathrm{train}}$ and $C_{\mathrm{train}}$, then
\begin{equation}
 \mathbb{E}_{x\sim\mathcal{D}_{\mathrm{train}}}[\logit p_x(1)]
 =F_{\mathrm{train}}(1)-\log\frac{C_{\mathrm{train}}}{k}.
 \label{eq:mean-log-odds}
\end{equation}
The effective total-Key reference $N_{\mathrm{train}}=C_{\mathrm{train}}+k$ generally differs from $\exp(\mathbb{E}[\log N_x])$.
The exact ratio $(N_t-k)/C_{\mathrm{train}}$ is approximated by total-length ratios only when $N$ is large relative to $k$.

This reference represents mean log-odds rather than mean attention mass. Without the training distribution, $N_0$ is a fixed experimental proxy (Appendix~\ref{app:reference-selection}), not a measurement of training-time attention.

The first-order Taylor expansion around $\beta=1$ underlying the fixed-slope approximation is
\begin{equation}
 F_{\mathrm{ref}}(\beta)\approx F_{\mathrm{ref}}(1)
 +\Delta_{\mathrm{ref}}(\beta-1),
 \qquad\Delta_{\mathrm{ref}}=F'_{\mathrm{ref}}(1)>0,
 \label{eq:linearization}
\end{equation}
which would give
\begin{equation}
 \beta_t\approx\min\left\{\beta_{\max},
 1+\frac{\log r_t}{\Delta_{\mathrm{ref}}}\right\}.
 \label{eq:linear-rule}
\end{equation}
Equation~\ref{eq:fixed-slope-rule} replaces $\Delta_{\mathrm{ref}}$ with the fixed slope $s$.
Its error depends on slope choice, the Taylor remainder, and current/reference contrast mismatch. The fixed-slope rule inherits the guarantee of Proposition~\ref{prop:mass} when its assumptions hold and the reference contrast is linear with slope $s$.

\section{Implementation and Evaluation Protocols}
\label{app:reproducibility}

\subsection{Query-Scaling Implementation}
\label{app:reference-selection}

The DFlash experiments use $k=32$, $s=5$, and $\beta_{\max}=e^{0.5}$, with $N_0=2048$ for primary requests and Qwen/Llama multi-turn sessions and $N_0=4096$ for GLM.
These are fixed experimental proxies, not estimates of the training-time visible-Key distribution.
GLM's setting is informed by a separate held-out MATH-500 replay comparison; it is not selected on the reference-sensitivity evaluation inputs.

\paragraph{Finite-precision application.}
\label{app:query-rescaling}
For the BF16 Q-normalization output $y$, the primary implementation applies
\begin{equation}
 y'=y+\operatorname{cast}_{\mathrm{BF16}}
       \bigl(\operatorname{float}(y)(\widehat{\beta}_t-1)\bigr).
 \label{eq:precision}
\end{equation}
The gain delta and product are computed in FP32, and the product is cast to BF16 before addition; normalization weights remain unchanged.

\paragraph{Cost accounting.}
\label{app:runtime-overhead}
In full attention, evaluating the scalar requires only context length and block size; rescaling is proportional to the current query tensor size.
Its arithmetic and temporary allocations are included in E2E time and peak memory, without a training copy, gradients, or optimizer state.

\subsection{End-to-End Evaluation}
\label{app:evaluation-metrics}

Table~\ref{tab:main-results} uses matched DFlash draft models with PyTorch SDPA on one NVIDIA RTX 6000 Ada GPU per run.
The cohort comprises 30 AIME-26 questions, all 198 GPQA-Diamond questions, and 100 LongGenBench Daily Diary prompts per model.
GPQA answer choices are permuted with seed zero in source order.
Speculative arms undergo a first-question token and finish-reason parity check against Frozen before independent generation of subsequent requests.
Methods share prompts, target--draft precision, and verification policy.
AR uses the same target checkpoint and precision without a draft model.

\paragraph{Online adaptation.}
Full Adaptation uses Adam with learning rate $5\times10^{-5}$ and momentum coefficients $(0.9,0.999)$.
LoRA uses AdamW with learning rate $10^{-3}$, momentum coefficients $(0.9,0.95)$, and zero weight decay.
Both use an update stride of five verification rounds.

\paragraph{Metrics.}
MAL is emitted tokens, including correction or bonus tokens, divided by target verification rounds.
E2E includes prefill, drafting, verification, calibration or learning, and synchronization.
Speedup is total reference E2E divided by total method E2E, with a geometric mean across benchmarks.
Tables~\ref{tab:main-results} and~\ref{tab:pard-results} both define E2E speedup relative to target-only AR on the same model, prompts, and hardware.
Peak memory is the maximum GPU allocation across the 228 AIME-26 and GPQA-Diamond requests.
Update cost is training-stream CUDA time divided by completed optimizer updates on these requests; zero denotes no optimizer updates.

\paragraph{Daily Diary prompts and stopping.}
\label{app:diary-protocol}
We use prompt ordinals 100--199 from the official long split, each requesting 365 entries.
Every request starts from fresh state with the raw prompt and BOS when defined, without a chat template or added reasoning instruction.
Greedy decoding stops at EOS, the first decoded \texttt{*** finished}, 32,000 new tokens, or the 40,960-token context limit.
The Llama \method{} run instead used a 32,768-token output cap, with seven outputs exceeding 32,000.
All outputs are retained, without repetition guards.
Runs use BF16, PyTorch 2.11.0+cu128, Transformers 5.5.4, and SDPA with flash, memory-efficient, and math backends enabled and cuDNN disabled.

\paragraph{Request-time visualization.}
\label{app:request-time-breakdown}
Figure~\ref{fig:request-time-breakdown} averages AIME-26 E2E times and per-request training-stream CUDA durations.
Training may overlap decoding: the right-aligned highlights neither locate execution intervals nor form an additive critical-path decomposition.
Drafting, verification, and calibration are not timed separately, and the unhighlighted portion is not a measured decoding component.

\subsection{Transfer Across Draft Model Families}
\label{app:cross-drafter}
Table~\ref{tab:pard-results} evaluates Qwen3-8B on H200 using 30 AIME-26 questions, 73 GPQA-Diamond questions, and 100 Daily Diary prompts.
For EAGLE 3.1, we use the Attention Drift authors' released checkpoint \texttt{Dogacel/paper-attention-drift-qwen3-post} at revision \texttt{6e6ee69bf2e72e0e0654425bdca06d7b097e52a9}~\citep{attentiondrift}.
Its configuration enables FC normalization and post-norm hidden-state feedback (\texttt{fc\_norm=true}, \texttt{norm\_output=true}).
PARD uses twelve proposals and $N_0=8192$; EAGLE 3.1 uses eight proposals and $N_0=2048$.
Both use $k=32$, $s=5$, and a log-gain cap of $0.5$.
The LoRA controls schedule updates every five verification rounds; the EAGLE 3.1 run starts at round five and permits up to 8,192 updates.

\subsection{Long-Generation Motivation}
\label{app:figure2-diagnostics}

Figure~\ref{fig:cardinality-motivation} uses Qwen3-8B on the first LongGenBench Diary prompt in our evaluation (source ordinal 100), separately from the aggregate latency evaluation.
The throughput measurements use an RTX 6000 Ada GPU, BF16 greedy decoding, and a 40,000-token output cap.
DFlash, PARD, and EAGLE 3.1 use frozen draft models; EAGLE 3.1 uses top-1 drafting.

\paragraph{Throughput.}
Panel (a) averages three fresh-process runs per method, including a separately measured AR baseline.
Decode rates use 512-token elapsed-time windows at a stride of 64 tokens, with timestamps interpolated within each run; prefill is excluded.
For display, each mean-rate curve, including AR, is smoothed over a centered 1,024-token span.
The linear vertical axis shows absolute decode throughput in tokens per second.
Methods may generate different output sequences, so positions represent generated-token counts rather than matched text prefixes.

\paragraph{Attention concentration.}
Panel (b) replays each draft model along its saved token sequence and samples top-32 attention mass every 64 generated tokens, averaging over draft layers, heads, and proposal queries.
Lines show centered 2,048-token local means; shaded bands show $\pm2$ population standard deviations of the sampled values within the same window.
Windows use only observed positions at the boundaries, without padding or extrapolation.
The bands describe local variation across token positions, not confidence intervals or uncertainty across repeated runs.
The draft architectures have different query scopes and visible-Key sets, so cross-method mass levels are descriptive.

\subsection{Multi-Turn Acceptance}
\label{app:multiturn-mal}

Figure~\ref{fig:multiturn-mal-profiles} uses a 223-question evaluation subset of MATH-500~\citep{lightman2023verify}, with thirty order seeds ($2026092600$--$2026092629$) shared across methods within each model.
The source is the \texttt{test.jsonl} file from \href{https://huggingface.co/datasets/HuggingFaceH4/MATH-500/tree/6e4ed1a2a79af7d8630a6b768ec859cb5af4d3be}{\texttt{HuggingFaceH4/MATH-500}} at revision \texttt{6e4ed1a2a79af7d8630a6b768ec859cb5af4d3be}.
Of its 500 questions, 262 have difficulty level 4 or 5 (128 and 134, respectively).
For each original row, \texttt{source\_id} is the literal prefix \texttt{math500:} followed by its \texttt{unique\_id}; if that field is absent, we append the decimal zero-based row index in the original file, before filtering or shuffling.
For example, \texttt{test/precalculus/881.json} becomes \texttt{math500:test/precalculus/881.json}.
We hash the UTF-8 encoding of \texttt{source\_id} with SHA-256, interpret the full hexadecimal digest as a base-16 integer $h$, and retain the question if $h\bmod 5\ne 0$.
This selects 223 evaluation questions and excludes the other 39.
For each seed, the full 262-question pool is shuffled before applying this split filter, and questions are not reused within a session.
EOS or 5,000 new tokens advances to the next question while retaining caches and adaptation state; state resets between sessions.
Repetitive sessions are retained, so MAL measures acceptance on the observed continuations rather than answer quality.

\paragraph{GLM checkpoints.}
The GLM target is \href{https://huggingface.co/zai-org/GLM-4.7-Flash/tree/7dd20894a642a0aa287e9827cb1a1f7f91386b67}{\texttt{zai-org/GLM-4.7-Flash}} at revision \texttt{7dd20894a642a0aa287e9827cb1a1f7f91386b67}.
Its DFlash draft is \href{https://huggingface.co/Dariusfar/GLM-4.7-Flash-DFlash/tree/d08993a8ec6fc0b50eb539bafa9f77a0de9a60ac}{\texttt{Dariusfar/GLM-4.7-Flash-DFlash}} at revision \texttt{d08993a8ec6fc0b50eb539bafa9f77a0de9a60ac}.
All thirty seeds and four methods initialize from this target--draft pair and use the same pinned MATH-500 source and evaluation split.

\paragraph{Generation and aggregation.}
Runs use H200, vLLM, BF16, and greedy decoding.
Qwen enables thinking with 16-token draft blocks; Llama uses its default template with 10-token blocks; GLM disables thinking with 16-token blocks.
Qwen/Llama use a 40,960-token context limit and output budgets of 40,000 tokens for ten seeds and 32,000 for twenty; panels (a)--(c) use the common first 32,000 tokens.
GLM uses a 100,000-token output budget and 131,072-token context limit.
Budgets exclude input tokens; context exhaustion may stop a session early.

Complete 2,048-token windows advance every 128 tokens and are plotted at their midpoints.
Rounds are assigned by their starting generated-token positions; each session contributes emitted tokens divided by verification rounds, and session ratios receive equal weight.
Incomplete windows are omitted without padding.
Bands use 20,000 paired seed-bootstrap resamples for pointwise 95\% percentile intervals.
Overlapping windows are correlated; late estimates use only contributing sessions.

\section{Analysis Protocols}
\label{app:analysis-protocols}

\subsection{Fixed vs.\ Dynamic Scaling}
\label{app:fixed-dynamic-live}

Table~\ref{tab:fixed-dynamic-length} uses Qwen3-8B/DFlash-b16 on all 30 AIME-26 questions with H200, BF16, PyTorch 2.11.0+cu130, and Transformers 5.5.4.
SDPA enables flash, memory-efficient, and math backends; cuDNN and TF32 are disabled.
Thinking is enabled, with greedy decoding to EOS or 32,768 new tokens and no repetition filtering.
Frozen, five fixed scales $\beta=e^{0.1j}$ ($j=1,\ldots,5$), and \method{} share the primary implementation and Equation~\ref{eq:precision}.
Each problem runs all seven conditions on the same GPU, with Frozen first and the remaining order rotated across problems.
Inclusion requires exact agreement in complete output tokens and termination across all conditions.
All \AblationPaired{} problems pass, with 26 EOS-terminated and four capped outputs per condition.

\paragraph{Aggregation and length groups.}
For $P$ paired problems, let $T_i$ be emitted tokens after the initial prefill token and $V_i$ the verification-round count.
The table averages whole-request ratios,
\begin{equation}
 \mathrm{MAL}_{\mathrm{overall}}=\frac{1}{P}\sum_{i=1}^{P}\frac{T_i}{V_i},
 \label{eq:ablation-overall-mal}
\end{equation}
rather than the pooled-round ratio $\sum_i T_i/\sum_i V_i$.
Final output length $L_i=T_i+1$ assigns each problem to exactly one group: $[0,10240)$, $[10240,20480)$, or $[20480,32768]$, including capped outputs in the last group.
The groups contain \AblationShortN{}, \AblationMediumN{}, and \AblationLongN{} problems, respectively.
Each group averages the same whole-request ratios; Overall is their problem-count-weighted mean.
Global and group-specific best fixed scales are selected on this cohort, not on held-out data.
These H200 acceptance measurements are separate from the RTX 6000 Ada latency results.

\subsection{Attention on Matched Prefixes}
\label{app:attention-sweep}

Figure~\ref{fig:dynamic-ablation-mass} probes the Frozen histories from Appendix~\ref{app:fixed-dynamic-live}, rebuilding target features in 512-token chunks.
At nominal $N\in\{2048,4096,8192,16384,32768\}$, it evaluates visible-Key counts $N-192,N-128,N-64,N$.
All seven settings share the prefix, target features, anchor token, and 15 masked proposal slots, without a persistent draft cache.
The scale uses the actual visible-Key count, including the 16-token draft block.

Attention logits and log-sum-exp reductions are reconstructed in FP32 from BF16 Queries and Keys.
Top-32 mass is averaged equally over five layers, 32 heads, and 15 proposal positions, then over four probes per problem and finally over eligible problems.
A problem contributes only when all four probes exist in every condition; terminated histories are not extended.
The five windows contain \AimeAttentionSupport{} problems.
Pointwise 95\% percentile intervals use 20,000 problem-bootstrap resamples paired across settings; the figure displays means only.
Instrumentation preserves draft hidden states bitwise at each problem's first and last probes.

\subsection{Cumulative Acceptance}
\label{app:cumulative-mal}

Figure~\ref{fig:dynamic-ablation-cumulative-mal} uses the live verification rounds from the same paired cohort at checkpoints $x\in\{1024,2048,\ldots,32768\}$.
Let $S(x)=\{i:L_i\geq x\}$ be the eligible problems, $\ell_{imr}$ the tokens emitted by method $m$ in round $r$, and $R_{im}(x)$ the number of complete rounds ending by $x$.
The equal-problem mean is
\[
 \overline{\mathrm{MAL}}_m(x)=\frac{1}{|S(x)|}\sum_{i\in S(x)}
 \frac{\sum_{r=1}^{R_{im}(x)}\ell_{imr}}{R_{im}(x)}.
\]
Emitted tokens include correction or bonus tokens but exclude the initial prefill token.
Rounds are not split or interpolated; the final included round may end up to 15 tokens before the checkpoint.
Long and capped outputs contribute at every earlier checkpoint, without extending terminated outputs.

The plot reports $100[\overline{\mathrm{MAL}}_m(x)/\overline{\mathrm{MAL}}_{\mathrm{dynamic}}(x)-1]$: a ratio of means over identical problems, not a mean of per-problem percentages.
The y-axis is linear within $\pm2$ percentage points and logarithmic outside.
The 1K/8K/16K/24K/32K checkpoints contain \AimeCumulativeSupport{} problems.
Pointwise intervals for absolute MAL use 20,000 paired problem-bootstrap resamples and are retained in the data; the figure shows relative means.
Evaluating each problem at its own endpoint recovers Table~\ref{tab:fixed-dynamic-length}; the 32K checkpoint instead uses only the \AimeCumulativeLateN{} problems reaching it.

\subsection{Fixed-Slope Sensitivity and Reference Validation}
\label{app:fixed-slope-validation}

\paragraph{Slope sweep.}
Figure~\ref{fig:slope-sensitivity} compares $s\in\{3,4,5,6,8\}$ on all 30 AIME-26 problems with Qwen3-8B/DFlash-b16, using native vLLM 0.27.1 in BF16 on H200.
Only $s$ changes, with $N_0=2048$, $k=32$, and $\beta_{\max}=e^{0.5}$ held fixed.
Each setting uses greedy decoding with thinking enabled and a 32,768-token output cap, stopping at EOS but not on repetition.
All 150 problem--setting pairs are complete, with one run per pair and no exclusions based on termination or repetition.
The line averages whole-request MAL equally across problems, and the band shows pointwise 95\% percentile bootstrap confidence intervals of this mean.
We use 20,000 resamples of the 30 problems (seed 20260926), sharing the sampled problem indices across slopes to preserve pairing.
These intervals quantify uncertainty across problems, not repeated-run variation or simultaneous coverage across slopes.
MAL excludes the initial prefill token and divides the remaining emitted tokens by verification calls.
Nine problems have identical outputs across all five settings, so this is a full-generation comparison rather than an output-matched comparison like Table~\ref{tab:fixed-dynamic-length}.
The mean fraction of post-reference verification rounds at the upper scale cap is $23.9\%$ for $s=3$, $2.8\%$ for $s=4$, and zero for $s\in\{5,6,8\}$.
Thus, clipping affects the smaller slopes and should not be mistaken for slope insensitivity.

\paragraph{Reference-curve diagnostic.}
To check whether a slope of five approximates the reference correction, we numerically invert a measured $F_{\mathrm{ref}}$ rather than using the deployed scaling rule.
The curve pools 19,200 Query rows from two Qwen3-8B/DFlash reference trajectories at $N_0=2048$, evaluated at 513 scales between $1$ and $e^{0.5}$.
For each $N>N_0$, we solve Equation~\ref{eq:contrast-correction} by interpolation to obtain $\beta^*(N)$, then compute the effective reference slope
\begin{equation}
 s_{\mathrm{eff}}^{\mathrm{ref}}(N)
 =\frac{\log[(N-k)/(N_0-k)]}{\beta^*(N)-1}
 =\frac{F_{\mathrm{ref}}(\beta^*(N))-F_{\mathrm{ref}}(1)}{\beta^*(N)-1}.
 \label{eq:effective-reference-slope}
\end{equation}
This is a secant slope, distinct from the local derivative $\Delta_{\mathrm{ref}}=F'_{\mathrm{ref}}(1)$.
The diagnostic uses 256-Key increments through 40K, excluding $N=N_0$ where the ratio is undefined.
All numerical solutions lie inside the measured scale range, without clipping or extrapolation.

Figure~\ref{fig:reference-effective-slope} shows that the inferred slope stays between $4.995$ and $5.112$ over this grid.
Thus, $s=5$ closely approximates the correction on this pooled reference curve.
This is not a distribution of live-state slopes or evidence of MAL optimality, since current-state contrast can differ from the reference.
The state-dependent offset is examined in Appendix~\ref{app:effective-slope}.

\begin{figure}[H]
 \centering
 \includegraphics[width=0.52\linewidth]{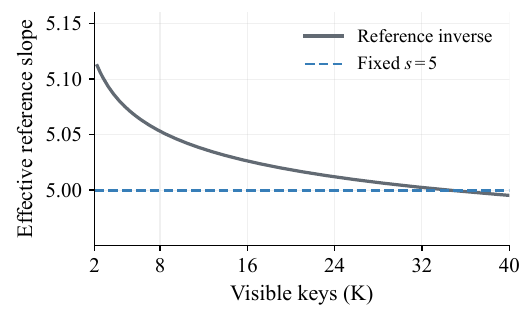}
 \caption{Effective slope inferred from the measured reference contrast for Qwen3-8B/DFlash. The dashed line marks the deployed setting $s=5$.}
 \label{fig:reference-effective-slope}
\end{figure}

\subsection{Layerwise Taylor Diagnostics}
\label{app:layerwise-taylor}

Figures~\ref{fig:layerwise-full-mass} and~\ref{fig:layerwise-full-mal} compare Frozen, the deployed fixed-slope rule, and Layerwise Taylor using Qwen3-8B/DFlash-b16 on H200 with native vLLM 0.27.1 and BF16.
Each of the 30 AIME-26 requests is run once per method, with greedy decoding, thinking enabled, and a 32,768-token output cap.
The analysis retains all requests, including capped and repetitive outputs.
Frozen, \method{}, and Layerwise Taylor have five, four, and five capped outputs, respectively, with detected tail cycles in three, three, and four requests.
Six requests have identical complete outputs across all three methods; the plotted profiles follow each method's full generation.

\paragraph{Layerwise expansion and diagnostic measurements.}
Layerwise Taylor uses the previous applied scale rounded to the nearest $0.1$ as its expansion center,
\begin{equation}
 \widehat\beta_{\ell,t}=\mathrm{clip}_{[1,2]}\!\left[
 c_{\ell,t}+\frac{F_{\mathrm{ref},\ell}(1)+\log r_t-F_{\ell,t}(c_{\ell,t})}
 {F'_{\ell,t}(c_{\ell,t})}\right],
 \qquad c_{\ell,t}=\operatorname{round}_{0.1}(\beta_{\ell,t-1}).
 \label{eq:layer-taylor-ablation}
\end{equation}
Here $F_{\ell,t}$ averages the contrast within layer $\ell$, and $r_t$ is the cardinality ratio from Equation~\ref{eq:fixed-slope-rule}.
The layerwise contrast averages 32 heads and 15 proposal Queries in each of five draft layers.
The reference contrasts and arithmetic-mean top-32 masses reuse two fixed SDPA calibration trajectories without refitting to vLLM.
Both scaling methods use $N_0=2048$ and $k=32$; the deployed scale is capped at $e^{0.5}$, while Layerwise Taylor is capped at $2$ and remains one below $N_0$.
Across the 4,450 post-reference layer--position diagnostic samples, the largest Layerwise Taylor scale is $1.916$ and 183 samples exceed $e^{0.5}$.
Each layer uses the hidden state produced by preceding scaled layers.
A separate diagnostic generation samples attention every 512 context tokens and reproduces the clean run's output tokens, verification rounds, and scales.
Taylor errors at center one and at the adaptive center are evaluated on the same Taylor states and applied scales, including finite-precision Query scaling and positional transformations.
Each expansion uses its measured local derivative, so this comparison isolates re-centering rather than the use of the fixed slope $s=5$.
Holding the state and applied scale fixed on the Taylor trajectories, mean absolute contrast error falls from $0.0433$ at center one to $0.0023$ with re-centering.
These diagnostics measure the centers visited during generation, not an independent sweep over fixed centers.

\paragraph{Output-matched timing comparison.}
The timing comparison in Section~\ref{sec:taylor-center} uses seven requests with identical output tokens and termination between \method{} and Layerwise Taylor, measured on the same physical GPU: 2026-I-01, 2026-I-04, 2026-I-05, 2026-I-07, 2026-I-09, 2026-II-01, and 2026-II-05.
Matching is required between these two methods, not against Frozen.
Six requests terminate naturally, while 2026-I-09 reaches the 32,768-token cap under both methods.
The additional output-matched request, 2026-II-12, uses different GPUs and is excluded only from this timing comparison.
Summed clean-run E2E times are $181.08$ seconds for \method{} and $218.95$ seconds for Layerwise Taylor, giving per-request means of $25.87$ and $31.28$ seconds ($20.9\%$ longer for Taylor).
The timing window includes prefill, decoding, calibration, and streaming delivery, but excludes model loading, reset, final trace serialization, and the separate diagnostic generations.
These measurements describe the current native-vLLM implementation rather than an isolated kernel-overhead benchmark or a repeated-run latency estimate.

\paragraph{Aggregation.}
Mass error is the absolute percentage-point deviation from each layer's reference top-32 mass, using samples with $N_t\ge2048$.
MAL includes all verification rounds.
Whole-request summaries weight all 30 requests equally.
For the figure, observations are first averaged within each request's 1,024-token generated-position bins, then across contributing requests.
Bands show pointwise 95\% percentile bootstrap confidence intervals of the mean, using 20,000 whole-request resamples (seed 20260926) shared across methods and positions.
Each resample averages only the requests observed in that bin; bins with one request have no band.
Terminated generations are not extended, so both the number and composition of contributing requests can change with position.
Points are placed at bin centers: the final point at 31.5K includes observations from the 31--32K bin, with five requests for Frozen, four for the deployed rule, and five for Layerwise Taylor.
Matching a layer's mean contrast does not require zero arithmetic-mean mass error, because the contrast and probability averages are different objectives.

\paragraph{Local slope and reference offset.}
\label{app:effective-slope}
At the single-head level, reference-mass matching requires $F_t(\beta_t)=F_{\mathrm{ref}}(1)+L_t$, where $L_t=\log r_t$.
Writing $A_t=F_{\mathrm{ref}}(1)-F_t(1)$ and $\Delta_t=F'_t(1)$, a first-order Taylor expansion gives $\beta_t-1\approx(A_t+L_t)/\Delta_t$.
Representing this correction as $L_t/\Delta_{\mathrm{eff},t}$ would require
\begin{equation}
 \Delta_{\mathrm{eff},t}=\frac{\Delta_t L_t}{A_t+L_t}.
 \label{eq:effective-slope-diagnostic}
\end{equation}
The expression is undefined when $A_t+L_t=0$ and can be nonpositive.
In the 30-request diagnostic, median layer-mean derivatives $F'_{\ell,t}(c_{\ell,t})$ on the Taylor trajectories are $4.84$, $4.39$, $5.26$, $4.72$, and $4.64$.
Effective denominators computed on Frozen states instead have medians $3.68$, $6.25$, $11.30$, $9.83$, and $7.69$.
These are measurements on different trajectories; they illustrate that a local derivative near five does not account for the current/reference offset or justify a universal effective slope.

\subsection{Reference-Cardinality Sensitivity}
\label{app:reference-sensitivity}

Figure~\ref{fig:mechanism}(a) uses all 30 AIME-26 questions for Qwen3-4B, Qwen3-8B, Llama-3.1-8B, and GLM-4.7-Flash.
Panel (b) reuses Qwen3-8B's AIME runs and adds the first 30 requests in the ordered GPQA73 selection and Diary source rows 100--129.
Each of the six model--dataset groups uses Frozen and $N_0\in\{1024,2048,4096,8192\}$ on identical inputs, retaining capped and repetitive outputs.

Runs use H200, vLLM 0.27.1, BF16, and greedy decoding; each setting generates independently from fresh single-turn state, stopping at EOS or 32,768 new tokens. Diary also stops at \texttt{*** finished}.
All Diary requests reach the cap in every arm.
The engine uses tensor parallelism one, one active request, eager execution, and no prefix caching or asynchronous scheduling.
Qwen reasoning prompts enable thinking, GLM disables it, and Llama uses its default template; Diary uses the raw primary prompt.
Draft blocks contain 16 tokens for Qwen/GLM and 10 for Llama.
The implementation counts persistent context plus 16 for every model, retaining $k=32$, $s=5$, and $\beta_{\max}=e^{0.5}$.

For setting $j$, let $V_{i,j}$ and $T_{i,j}$ be verification calls and generated tokens after prefill.
The pooled call rate and savings relative to reference $b$ are
\begin{equation}
 R_j=\frac{\sum_i V_{i,j}}{\sum_i T_{i,j}},
 \qquad
 S(j;b)=100\left(1-\frac{R_j}{R_b}\right).
 \label{eq:reference-call-rate}
\end{equation}
Reference $b$ is fixed at 2K for Qwen/Llama and 4K for GLM, rather than selected from the sweep.
Pointwise 95\% percentile intervals use 20,000 paired input-case bootstrap resamples, recomputing both totals in each draw; the figure displays point estimates.
The y-axis is linear within $\pm2$ percentage points and logarithmic outside, without omitting any range.

\end{document}